\documentclass[letterpaper,conference]{IEEEtran}
\IEEEoverridecommandlockouts

\usepackage{cite}
\let\citet\cite
\let\citep\cite

\usepackage{amsmath,amssymb,amsfonts}
\usepackage{graphicx}

\usepackage{amsmath, amsfonts, amssymb, amsthm, amscd, mathtools}
\usepackage[T1]{fontenc}
\usepackage{subcaption}

\newtheorem{thm}{Theorem}[section]

\newtheorem{lem}[thm]{Lemma}
\newtheorem{cor}[thm]{Corollary}
\newtheorem{remark}{Remark}[section]
\newtheorem{assump}{Assumption}[section]
\newtheorem{example}{Example}[section]

\def\RR{\mathbb{R}}

\def\PP{\mathbb{P}}

\def\dKL{d_{KL}}
\def\diag{\textsf{diag}}
\def\on{\mathbf{1}}
\def\EE{\mathbb{E}}
\def\YY{\mathcal{Y}}
\def\XX{\mathcal{X}}
\def\LL{\mathcal{L}}
\def\BB{\mathcal{B}}
\def\qbest{q^{\text{best}}_{n, m}}
\def\qbestn{q^{\text{best}}_{n, 1}}
\def\bestofn{best-of-$n$}
\def\Qone{(\textit{Q1})}
\def\Qtwo{(\textit{Q2})}

\DeclareMathOperator*{\argmax}{argmax}

\let\subsubsubsection\subsubsection

\begin{document}

\title{The Asymptotics of Language Model Alignment with Memory}

\author{
\IEEEauthorblockN{
Haricharan Balasundaram$^*$}
\IEEEauthorblockA{
School of Electrical and Computer Engineering, \\
Georgia Institute of Technology, Atlanta, GA\\
Email: haricharanb@gatech.edu
}
\and
\IEEEauthorblockN{V. Arvind Rameshwar}
\IEEEauthorblockA{
Department of Electrical Engineering,
\\
IIT Madras, Chennai, India\\
Email: arvind@ee.iitm.ac.in
}
\thanks{$^*$This work was carried out in part when the author was a student at IIT Madras.}
}

\maketitle

\begin{abstract}
Language model (LM) alignment broadly aims to perturb a given LM $Q$ into an \textit{aligned} LM $q$ such that i) the outputs produced by $q$ and $Q$ are `close' in probability, ii) $q$ has a higher expected \textit{reward} than $Q$. Two common techniques for LM alignment are: KL-constrained RL, which requires knowledge of the LM distribution and is computationally expensive, and the \bestofn{} algorithm, which requires only sampling from the LM. The work of \cite{asymptotics_of_LM_alignment} established asymptotic closeness between the distributions produced by the two alignment methods for an $m$--length \textit{i.i.d.} token sequence output by the LM, in the limit as $m$ increases to infinity. However, the i.i.d. assumption is not representative of practical LMs, whose output sequences often have memory. In this paper, we extend the asymptotic closeness result to the case when the $m$--length token sequence outputted by the LM is \textit{Markovian}. Further, for finite-length output sequences -- particularly, when $m=1$ -- we provide a complete characterization of LM distributions and reward functions for which the KL-divergence between the distributions produced by the two alignment methods is zero -- a question first posed in~\cite{asymptotics_of_LM_alignment}.
\end{abstract}

\section{Introduction}
\label{sec:introduction}

Large Language Models (LLMs) have become a central paradigm for generating natural language text and LLM-based systems are increasingly being adopted in practice \citep{llm_reference_2,llm_reference_1,llm_reference_3}. Given a prompt $x$, an autoregressive LLM defines a conditional probability distribution $\PP(\cdot|x)$ over output token sequences, with $\PP(s|x)$ denoting the probability assigned to the response $s$. LLMs are pretrained on large text corpora and may subsequently be adapted through task-specific fine-tuning or preference-based post-training \citep{llm_alignment_need_1,alignment_4}. 

LLMs sometimes generate unrealistic and undesirable outputs. For instance, the outputs might contain factually incorrect information or reflect biases in the training data \citep{llm_alignment_need_1,llm_alignment_need_2}. Sometimes, more specific properties of the generated output text, including fluency, linguistic style, and the avoidance of undesirable repetition, are targeted to be improved~\citep{ficler2017style,dathathri2020pplm,holtzman2020degeneration,welleck2020unlikelihood}. Desirable outputs are associated with a large `reward' and undesirable outputs are associated with a small `reward'. It is thus desirable to perturb the LLM probability distribution to an \textit{aligned} distribution that has a higher expected reward. 

While several techniques for LLM alignment~\citep{alignment_1,alignment_2,alignment_3,alignment_4,alignment_5,best_of_n_policy,variational_best_of_n} exist in the literature, two of the most popular techniques, which are also grounded in theory, are i) KL-constrained reinforcement learning (RL), and ii) \bestofn{} alignment. KL-constrained RL returns an aligned distribution (which we sometimes call the ``KL-constrained RL distribution'') whose reward is maximized while ensuring that the KL-divergence between the reference LLM distribution, i.e., the LLM distribution \textit{before} alignment, and the aligned LLM distribution is small. In \bestofn{} alignment, the algorithm instead samples $n$ times from the reference LLM distribution, picking the maximum-reward output among $n$ outputs generated. Unlike KL-constrained RL that returns a deterministic, analytical solution, the \bestofn{} algorithm returns a random response, the distribution of which we simply call the ``\bestofn{} distribution''.

In the light of analytical results on the \bestofn{} algorithm, the work of~\citet{asymptotics_of_LM_alignment} establishes a connection between the KL-constrained RL and the \bestofn{} distributions in the restricted setting where: i) the LLM generates sequences consisting of $m$ i.i.d.~tokens as the response for each prompt, and ii) the reward of the length-$m$ sequence can be decomposed as the sum of the reward over individual tokens. Under these assumptions, the KL-constrained RL and the \bestofn{} distributions, when $n = \Theta(\exp(m))$, are asymptotically close as $m$ grows to infinity. In other words, the optimal KL-constrained RL policy, which requires \textit{knowledge of the reference LLM distribution}, is asymptotically close to the computationally efficient \bestofn{} policy, which requires no such knowledge.

While the purpose of the two assumptions is to derive analytical connections between the two alignment methods for a simple, toy setting, the assumptions are not representative of the outputs of practical LLMs, which display a Markov dependence of responses conditioned on previous contexts~\citep{llm_markov_chain}. An interesting question to ask, in this light, is: \textit{Is the connection between the two alignment methods a consequence of the i.i.d. nature of the tokens, or does it hold even in the presence of memory among outputs?} A related question about removing i.i.d. LLM output assumption in~\cite{asymptotics_of_LM_alignment} was posed in the same paper. We take up this question in our work, analyzing the setting of Markov LLM output distributions.

Furthermore, the work of~\citet[Section 5]{asymptotics_of_LM_alignment} also observed empirically that when the LLM distribution and the reward function are both randomly generated, the KL-divergence between the distributions produced by the two alignment methods is ``small'' even for ``small'' $m$ values, when the LLM outputs are i.i.d. A second question to be asked in this regard, also posed in~\citet{asymptotics_of_LM_alignment} is: \textit{For small $m$ values, what LLM distributions and reward functions give rise to small KL-divergence between the distributions returned by the two alignment methods? In particular, what is the dependence of this KL-divergence on the vocabulary size?} We spend the latter part of this paper providing partial progress on this question, when $m = 1$.

\noindent \textbf{Our Contributions.} Our main contributions are hence an affirmative answer to the first question in the presence of memory, i.e., when the LLM outputs are Markov and when the reward function can be decomposed as the sum of local rewards of contexts, there is indeed an asymptotic closeness result between the KL-constrained RL and \bestofn{} distributions. We mention that such local reward criteria arise naturally in controlled decoding and token-level objectives, including attribute control and repetition penalties \citep{dathathri2020pplm,krause2021gedi,welleck2020unlikelihood}. When $m = 1$, we also identify classes of LLM reference distributions and reward functions for which the distributions produced by the two alignment methods are close in KL--divergence. In addition, for certain special classes, we characterize how the KL--divergence between the two alignment methods depends on the vocabulary size, thereby making progress on the second question as well.

Due to the paucity of space, some proofs are relegated to the appendices.

\section{Problem Formulation and Our Results}
\label{sec:notation_and_problem_formulation}

\subsection{Notation and Assumptions}
\label{sec:notation_and_assumptions}

Let $x \in \XX$ be an input (prompt) to a language model (LM). Suppose that the vocabulary (set of possible output tokens of the language model) is $V$ and that the LM is restricted to output a sequence of $m + 1$ tokens. Conditioned on the input $x \in \XX$, a language model (denoted by $M$) assigns a probability for each response, i.e., for each sequence $z^{m + 1} = (z_0 , z_1, \dots, z_m)$ of length $m+1$. In other words, an LM is characterized by the probability distribution
\begin{equation}
    \PP_M(z^{m+1} | x), \qquad z^{m+1} \in V^{m + 1},\ x \in \XX.
\end{equation}

Both classical $k$--gram LMs~\citep{chen_goodman} and modern-day LLMs generate the next token probabilistically depending on a context consisting of the previous $k - 1$ tokens, in which case the probability of the given sequence $z^{m+1}$ conditioned on the input $x$ can be decomposed as follows: 
\begin{equation}
    \PP_M(z^{m+1} | x) = \prod_{i = 0}^{m} \PP_M(z_{i} | z_{i - k + 1}, \dots, z_{i - 1}, x),
    \label{eqn:Markov_reward}
\end{equation}
where the sequence is appropriately padded with start tokens to ensure that $z_{-1}, z_{-2}, \dots$ are well-defined. 

Note that~\eqref{eqn:Markov_reward} essentially establishes the stationary Markov nature of LMs. Formally, for a \textit{fixed} prompt $x$, we absorb $x$ into $y_{i - 1}:= (z_{i - k + 1}, \dots, z_{i - 1}, x) \in \YY$, where $\YY = \XX \times V^{k - 1}$, to obtain 
\begin{equation}
    \PP_M(z^{m+1} | x) = \PP_M(y^{m+1}) = \PP(y_0) \prod_{i = 1}^{m} \PP(y_{i} | y_{i - 1}),
\end{equation}
where $y_{0}$ is defined to be a start token and we define $Q(y^{m+1}):= \PP_M (y^{m+1})$ since the model $M$ is fixed. This procedure is formally stated in Assumption~\ref{assump:Markovness}.

Let $\Delta_a$ denote the probability simplex on a size--$a$ alphabet. 
\begin{assump}
\label{assump:Markovness}
We assume that the LM outputs a sequence $y^{m+1} = (y_0, y_1, \dots, y_m)\in \YY^{m+1}$ according to a Markov process with an irreducible aperiodic transition matrix $P_{|\YY| \times |\YY|}$, with stationary distribution $\pi \in \Delta_{|\YY|}$. We also assume that the first element of the sequence, $y_0$, is drawn from $\pi$. Thus, the probability distribution of the sequence $y^{m+1}$ is obtained as:
\begin{equation}
    \PP_M(y^{m+1}) = Q(y^{m+1}) = \pi_{y_0} \prod_{i = 1}^{m} P_{y_{i - 1} y_{i}}.
    \label{eqn:Q_definition}
\end{equation}    
\end{assump}

For each sequence outputted by the LM for a given input (or prompt) $x$, we assume there exists a positive-valued reward function $r: (\XX \times V^{m+1}) \to \mathbb{R}_+$. Equivalently, we write the reward as $r: \YY^{m+1} \to \mathbb{R}_+$.

To analyze the behavior of the distribution which maximizes the reward, we utilize Assumption~\ref{assump:reward_decomposition}, corresponding to the local decomposition of the reward function discussed earlier.
\begin{assump}
\label{assump:reward_decomposition}
The reward function decomposes as the sum of the rewards of adjacent terms, i.e. 
\begin{equation}
    r (y^{m+1}) = \sum_{i = 1}^{m} r (y_{i - 1}, y_{i}) = \sum_{i = 1}^{m} r_{y_{i - 1} y_{i}},
\end{equation}
where $r_{ab}$ is a shorthand for $r(a, b): \YY \times \YY \to \mathbb{R}_+$.
\end{assump}

\subsection{Problem Formulation}
\label{sec:problem_formulation}

\textbf{KL-constrained RL:} We attempt to obtain an aligned LM which is a probability distribution $q$, such that i) the distance, measured by the KL-divergence (as in \citet{KL_constrained_RL}) $d_{KL}(q,Q)$, between the probability distributions $q$ and $Q$ is at most $m\delta$ for some $\delta>0$, and ii) the expected reward of $q$ taken over $m+1$ samples drawn from $q$, is \textit{maximized}. Note that this expected reward is:
\begin{equation}
    \EE_{Y^{m+1} \sim q}[r(Y^{m+1})] = \sum_{y^{m+1} \in \YY^{m+1}} \hspace{-0.3cm} q(y^{m+1}) r(y^{m+1}).
\label{eqn:expected_reward_expression}
\end{equation}
We require that $q$ be a valid probability distribution, i.e., $q \in \Delta_{|\YY|^{m+1}}$. Thus, the KL-constrained RL can be formally stated as finding
\begin{multline}
q_{\beta} = \argmax_{q \in \Delta_{|\YY|^{m+1}}} \sum_{y^{m+1} \in \YY^{m+1}} q(y^{m+1}) r(y^{m+1}),\\\text{ subject to } \dKL(q, Q) \le m \delta. \label{eqn:optimization_problem} \tag{KL-RL}
\end{multline}

\textbf{Best-of-$n$:} Let $Y^{m + 1, i} \sim Q$ be a length-$(m+1)$ sequence output by the LM according to the distribution $Q$ for some fixed $i$. Now, consider a total of $n$ such length-$(m+1)$ sequences from $Q$, which we call $Y^{m+1, 1}, Y^{m+1, 2}, \dots, Y^{m+1, n}$. The \bestofn{} algorithm sets
\begin{equation}
    K^* = \argmax_{i \in [n]}\ r(Y^{m+1, i}) \label{eqn:best-of-n_definition}
    \tag{Bo-$n$}
\end{equation}
and returns $Y_{K^*}$. Essentially, the algorithm outputs the highest-reward (`best') sequence from $n$ LM outputs. We let the probability distribution of this random output $Y_{K^*}$ be denoted as $\qbest$.

The two questions in Section~\ref{sec:introduction} can be formalized as:\subsubsubsection{\Qone{}} When the LM outputs are drawn from a Markov chain and the rewards decompose locally (Assumptions~\ref{assump:Markovness} and~\ref{assump:reward_decomposition}), do we still have $\lim_{m \to \infty}\frac 1m \dKL(\qbest, q_\beta) = 0\ ?$
\subsubsubsection{\Qtwo{}} Assume that $m = 1$. For what choice of the LM distribution $Q$ and reward function $r$ is it the case that $\dKL(\qbest, q_\beta)$ either equals or is close to zero?

\subsection{Our contributions towards \Qone{}}
\label{sec:results_Q1}

We first prove the theorem below, which connects the KL-constrained RL solution with the exponentially \textit{tilted} distribution studied in \citet{salamatian_tilted_distribution}. We define the matrix $G_\beta$ of the same dimensions as $P$ to be such that 
\begin{equation}
    (G_\beta)_{ij} = P_{ij} \cdot \exp(\beta r_{ij})
\label{eqn:G_beta_definition}
\end{equation}
\begin{thm}
\label{thm:optimization_problem_q_solution}
Under Assumptions~\ref{assump:Markovness}~and~\ref{assump:reward_decomposition}, the solution to~\eqref{eqn:optimization_problem} is    
\begin{equation*}
    q_{\beta}(y^{m + 1}) = \frac{\displaystyle \pi_{y_0} \prod_{i=1}^m {(G_\beta)}_{y_{i - 1} y_i}}{\displaystyle \sum_{\substack{z^{m + 1} \\\in \YY^{m+1}}} \pi_{z_0} \prod_{i=1}^m {(G_\beta)}_{z_{i - 1} z_i}},\ y^{m + 1} \in \YY^{m+1},
\end{equation*}
where $\beta \ge 0$ is chosen such that 
\begin{equation}
    \dKL(q_{\beta}, Q) = m \delta.\label{eqn:beta_value_choice}
\end{equation}
\label{thm:q_expression}
\end{thm}
\vspace{-0.6cm}
The proof of Theorem~\ref{thm:optimization_problem_q_solution} can be found in Section~\ref{sec:KL-constrained_RL}. $G_{\beta}$ is, in general, \textit{not} a stochastic matrix. However, as the next theorem shows, one can write $q_{\beta}$ as a non-stationary Markov process.
\begin{thm}[Concise]
    There exist an initial distribution $\nu$ and $m$ stochastic matrices $p^{(1)}, p^{(2)}, \dots, p^{(m)}$ such that
\begin{equation}
    q_{\beta}(y^{m + 1}) = (\nu)_{y_0} \prod_{t=1}^m p^{(t)}_{y_{t - 1} y_t}.  
\end{equation}    
\label{thm:q_simplified_expression_informal}
\end{thm}
Explicit characterizations of $\nu$ and the $p^{(t)}$'s can be found in Theorem~\ref{thm:q_simplified_expression}.

Interestingly, while $\nu$ and the $p^{(t)}$ distributions characterize a non-stationary Markov process, for large values of $m$, this Markov process approaches stationarity and $p^{(t)} \approx p^{(t - 1)}$ for sufficiently small $t$.
\begin{thm}[Concise]
For $1 \le t_1, t_2 \le m$ such that $\lim_{m \to \infty} (m - t_1) \to \infty$ and $\lim_{m \to \infty} (m - t_2) \to \infty$, we have
$
    \lim_{m \to \infty} p^{(t_1)} = \lim_{m \to \infty} p^{(t_2)}.
$
\label{thm:p_convergence_informal}
\end{thm}
The proof of Theorem \ref{thm:p_convergence_informal} and more discussion about this long-term behavior of the $p^{(t)}$'s is in Section~\ref{subsec:KL-constrained_RL_long_term}.

Next, we state a large deviation principle (LDP) for the average reward under $Q$. Let $Y^{m + 1} \sim Q$ and $S_m(Y^{m + 1}):= \frac{1}{m}\cdot r(Y^{m+1})$.
\begin{thm} The random variable $S_m$ satisfies an LDP with rate
\begin{multline*}
\tilde{J} (s) = \min_{\{\theta_{ij},\theta_i\}} \Bigg\{\sum_{i, j \in \YY} \theta_i b_{ij} \log \left(\frac{b_{ij}}{P_{ij}} \right) : \\\sum_{i, j \in \YY} r_{ij} \theta_i b_{ij} = s, \quad \sum_{i \in \YY}\theta_{ij} = \theta_j \Bigg\},
\end{multline*}    
where $b_{ij} = \frac{\theta_{ij}}{\theta_i}$. Further, $\tilde{J} (s)$ is a {good} rate function.
\label{thm:rate_kl_value}
\end{thm}
\vspace{-0.3cm}
The definition of a good rate function above is standard in the LDP literature (see, for instance, \citet[Section 1.2]{ldp_book}). The proof of Theorem \ref{thm:rate_kl_value} is in Section~\ref{sec:LDP_average_reward}.

Finally, we show that the distributions produced by the two alignment methods are asymptotically close, as $m$ tends to infinity. Note that the KL divergence between the \bestofn{} distribution and the reference LM distribution, $\dKL(\qbest, Q)$, is at most $\log (n)$ (as in \citep[Theorem 3.1]{best_of_n_policy}). In particular, by taking $n = \exp(m \delta)$, we get $\dKL(\qbest, Q) \le m \delta$. The \bestofn{} distribution thus maximizes the \textit{empirical} reward while implicitly maintaining $\dKL(\qbest, Q) \le m \delta$, while the tilted distribution formulation\footnote{Recall that $q_\beta$ is also a function of $m$ but we omit that for brevity.} for $q_\beta$ explicitly includes a KL-constraint between the tilted distribution and the LM distribution (i.e. $\dKL(q_\beta, Q) \le m \delta$) while maximizing the reward. It is thus somewhat intuitive to expect that the normalized KL distance between the tilted and \bestofn{} distributions asymptotically goes to $0$. We formalize this intuition in the following theorem.
\begin{thm}
For $n = \exp(m \delta)$, we have
\begin{equation} \lim_{m \to \infty}\frac 1m \dKL(\qbest, q_\beta) = 0.    
\end{equation}
\label{thm:KL-constrained_best-of-n}
\end{thm}
The proof of Theorem \ref{thm:KL-constrained_best-of-n} can be found in Section~\ref{sec:connection_KL_best-of-n}.

\subsection{Our contributions towards \Qtwo{}}
\label{sec:results_Q2}

As shown in Theorem \ref{thm:KL-constrained_best-of-n}, the KL-constrained RL and the \bestofn{} distributions are asymptotically close when the length-$(m + 1)$ output of the LM is Markov with mild assumptions (or i.i.d., as in~\citep{asymptotics_of_LM_alignment}) and when $n = \exp(m \delta)$. However, this ``closeness'' result holds only when $m \to \infty$. As we argue in this section, for certain classes of distributions and reward functions, the KL-divergence under consideration could be small even if for finite, small $m$ values.

In what follows, we take $m = 1$ and attempt to characterize such distributions and (associated) reward functions. Fix some $\delta > 0$. Now, consider the tilted distribution $q_\beta$ and the \bestofn{} distribution, $\qbestn$, with
\begin{equation}
    \delta = \log (n) - \frac{n - 1}{n}
    \label{eqn:delta_n_relation}
\end{equation}
as in~\citet[Theorem 3.1]{best_of_n_policy} and both distributions supported on a set $\YY$. Our definition of ``smallness'' of KL-divergence between $q_\beta$ and $\qbestn$ is that the KL-divergence be \textit{exactly zero}; in other words, we want $q_\beta = \qbestn$ almost surely. We assume that independent samples from $Q$ have distinct rewards almost surely as in \citet{best_of_n_policy}. Our interest is in understanding which choices of the initial LM distribution $Q$ and reward $r$ achieve this. We first discuss an example for which this KL-divergence is zero.

\begin{example} \label{example:KL_alignment_zero}
Suppose that $Y$ has the p.d.f. $Q(y) = 1,\ y \in (0, 1]$. Let
\begin{equation}
    r(y) = b + \frac{n - 1}{\beta}\cdot \log(y),\  y \in (0, 1],
\end{equation}
where we recall that $\beta$ is such that $\dKL(q_{\beta}, Q) = \delta$.

The KL-constrained RL distribution $q_\beta$ obeys, for a suitable normalization constant $C > 0$,
\begin{align}
    q_\beta(y) & = C \cdot 1 \cdot \exp \left(\beta \left(b + \frac{n - 1}{\beta} \log(y)\right) \right)\\
    & = C \cdot \exp \left(\beta b \right) y^{n - 1} \overset{(a)}{=} n y^{n - 1} \ \forall 0 < y \le 1,
\end{align}
where $(a)$ follows since $q_\beta$ must integrate to $1$. It is also possible to show that
$
\dKL(q_{\beta}, Q) \le \log (n) - \frac{n - 1}{n} = \delta,
$
where the equality follows from~\eqref{eqn:delta_n_relation}. Finally, note that the \bestofn{} distribution is simply the maximum of $n$ i.i.d. random variables, each drawn from $Q(y)$. Thus,
\begin{equation}
    \qbestn(y) = n y^{n - 1} ,\ 0 < y \le 1.
\end{equation}
Clearly, $q_\beta$ equals $\qbestn$, and $\dKL(q_\beta,\qbestn) = 0$.   
\end{example}

In the following theorem, we provide a general result on classes of $Q$ and $r$ such that the $\dKL(q_\beta,\qbestn) = 0$. Suppose that the random variable $Y$ has the distribution $Q$ and let $L = r(Y)$, and let $\LL$ denote the set of values that $L$ can take. Further, let $p_L$ denote the PDF of $L$, with associated CDF $P_L$.
\begin{thm}
Suppose $n\ge 2$ and the rewards are distinct almost surely. Then,
$
\qbestn = q_\beta    
$
almost surely if and only if $\exists \beta > 0,\ b\in\mathbb R$ such that
\begin{equation}
P_L(l) = \begin{cases}
\exp \left(\frac{\beta}{n - 1}(l-b) \right), & l\le b,\\
1, & l>b.
\end{cases}
\label{eqn:continuous_exponential_form}
\end{equation}
Further, the value of $n$ must satisfy $\log (n) - \frac{n - 1}{n} = \delta$, where we recall that $\beta$ in~\eqref{eqn:continuous_exponential_form} is picked satisfying $\dKL(q_\beta, Q) = \delta$ from~\eqref{eqn:beta_value_choice}.
\label{thm:continuous_KL_separation}
\end{thm}

Thus, if the CDF of the reward $L$ is a suitably shifted exponential, the KL-divergence between the KL-constrained RL and \bestofn{} distributions is zero.

Note that the above theorem implies that the KL divergence between the distributions produced by the two alignment methods is zero only in the continuous case when $m = 1$. However, typical LM output distributions are \textit{discrete}. Thus, we consider the discrete analogue of Example~\ref{example:KL_alignment_zero}, and provide the following theorem for the dependence of the KL separation on the vocabulary size, thus shedding light on a question raised in~\citet{asymptotics_of_LM_alignment} on the relevant KL-divergence for small $m$ values.

Consider a vocabulary of size $d$, defined over the words in $\YY_d = \{\frac 1d, \frac 2d, \dots, \frac dd\}$ and consider the uniform distribution $Q$ on $\YY_d$ with $Q(y) = \frac 1d,\ \forall y \in \YY_d$. For finite $d$, the reward function $r(y) = b + \frac{(n - 1)}{\beta}\cdot \log(y),\ \forall y \in (0, 1]$ which makes the \bestofn{} and tilted distributions the same in the continuous case, no longer makes the KL divergence between the two distributions equal to zero in the setting of a uniform, but discrete, distribution $Q$. However, we consider the same reward function even for the discrete case, and argue that the KL-divergence of interest in fact decays to zero for large enough vocabulary size $d$ in the following theorem.
\begin{thm}
For $n \ge 4$, when $Q$ is uniform over $\YY_d$ and $r(y) = b + \frac{(n - 1)}{\beta} \log(y)$, for $y\in (0,1]$, we have 
\begin{equation}
\dKL(\qbestn, q_\beta) = \frac{n}{8(n-2)}\frac1{d^2}+O\left(\frac{1}{d^3}\right).
\end{equation}
Further, $q_\beta$ satisfies $\dKL(q_\beta, Q) \le \delta$, as in~\eqref{eqn:beta_value_choice}.
\label{thm:discrete_KL_separation}
\end{thm}
The proof of Theorems~\ref{thm:continuous_KL_separation} and~\ref{thm:discrete_KL_separation} and additional, related discussion can be found in Section~\ref{sec:small_m_connection_KL_best-of-n}.

\begin{figure}[!htbp]
    \centering
    \includegraphics[width=0.8\linewidth]{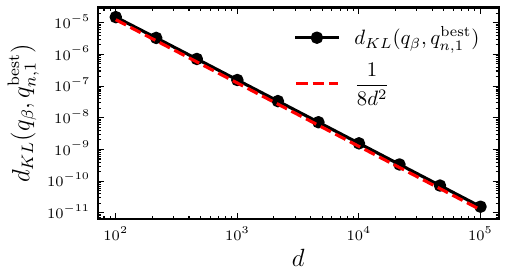}
    \caption{$\dKL(q_\beta, \qbestn)$ and the theoretical asymptotic rate $\frac{1}{8d^2}$ vs $d$ for $n = 10$ from Theorem~\ref{thm:discrete_KL_separation}.}
    \label{fig:KL_separation_vs_d}
\end{figure}

In Figure~\ref{fig:KL_separation_vs_d}, we verify that the theoretical rate from Theorem~\ref{thm:discrete_KL_separation} is indeed correct by plotting the calculated values of $\dKL(\qbestn, q_\beta)$ against $d$ and comparing it against the theoretical (asymptotic) rate of $\left( \frac{1}{8d^2} \right)$.

\section{KL-Constrained RL Solution}
\label{sec:KL-constrained_RL}

We first consider the solution to the optimization problem~(\ref{eqn:optimization_problem}), thereby proving Theorem~\ref{thm:q_expression}.

\begin{proof}[Proof of Theorem~\ref{thm:q_expression}] It is possible to write a Lagrangian to obtain~\eqref{eqn:Q_general}, analogous to~\citet[Lemma 1]{asymptotics_of_LM_alignment}. The solution to~\eqref{eqn:optimization_problem} in general is the \textit{tilted} distribution \citep{salamatian_tilted_distribution} for $y^{m + 1} \in \YY^{m+1}$:
\begin{equation}
    q_{\beta}(y^{m + 1}) = \frac{Q(y^{m + 1}) \exp(\beta r(y^{m + 1}))}{\sum_{z^{m + 1} \in \YY^{m+1}} Q(z^{m + 1}) \exp(\beta r(z^{m + 1}))}.
    \label{eqn:Q_general}
\end{equation}
Next, we use Assumptions~\ref{assump:Markovness} and~\ref{assump:reward_decomposition} and the expression for $Q$ from~\eqref{eqn:Q_definition} to obtain 
\begin{equation}
q_{\beta}(y^{m + 1})
= \frac{\pi_{y_0} \prod_{t=1}^m (P_{y_{t - 1} y_t} \exp(\beta r_{y_{t - 1} y_t}))} {\displaystyle \sum_{\substack{z^{m + 1}\\\in \YY^{m+1}}} \pi_{z_0} \prod_{t=1}^m (P_{z_{t - 1} z_t} \exp(\beta r_{z_{t - 1} z_t}))}. \nonumber
\end{equation}
Substituting $G_\beta$ from~\eqref{eqn:G_beta_definition} completes the result.
\end{proof}

We next state the formal version of Theorem~\ref{thm:q_simplified_expression_informal} with $\nu$ and the $p^{(t)}$ expressions defined explicitly. Let $h^{(t)} \in \RR^{|\YY|}$ for $0 \le t \le m$, be defined recursively as follows:
\begin{align}
    h^{(m)} &= \on_{|\YY|} \label{eqn:h_t_definition} \\ 
    h^{(t)} & = G_{\beta} h^{(t + 1)} = G_{\beta}^{m - t} \on_{|\YY|},\ 0 \le t \le m - 1 \nonumber,
\end{align}
where $\on_d$ denotes the vector of all ones in $\RR^d$.
\begin{thm}
We have
\begin{equation}
    q_{\beta}(y^{m + 1}) = (\nu)_{y_0} \prod_{t=1}^m p^{(t)}_{y_{t - 1} y_t},\ \forall y^{m+1}\in \YY^{m+1},
\end{equation}
\vspace{-0.3cm}
\begin{align*}
    \text{where } \ \nu_i & = \frac{\pi_i {h_i^{(0)}}}{\pi^\top (G_\beta)^m \on_{| \YY |}}\ \text{and}\\
    p^{(t)}_{ab} & = {(G_{\beta})}_{ab} \frac{h^{(t)}_b}{h^{(t - 1)}_a},
    \ 1 \le t \le m; a,b\in \YY.
\end{align*}
Furthermore, $\nu$ is a valid probability distribution and $p^{(t)}$, $ 1 \le t \le m$, are valid stochastic matrices.
\label{thm:q_simplified_expression}
\end{thm}
The proof of this theorem involves careful algebraic manipulations and is relegated to Appendix~\ref{appendix:q_expression}.

\subsection{Long-term behavior}

\label{subsec:KL-constrained_RL_long_term}

In this subsection, we prove Theorem~\ref{thm:p_convergence_informal}, and discuss a collection of other results pertaining to the optimization problem \eqref{eqn:optimization_problem}. We first show that the matrix to which the $p^{(t)}$ probabilities converge to in Theorem~\ref{thm:p_convergence_informal}, which we call $p$, is the same as the solution to the optimization problem in~\eqref{eqn:optimization_problem} but with an additional ``flow'' constraint added.

Formally, instead of using $q$ as the optimization variable in \ref{eqn:optimization_problem}, we formulate an equivalent optimization problem over the variables $(\mu, p),\ \mu \in \RR^{|\YY|}, p \in \RR^{|\YY| \times |\YY|}$, such that the flow constraint $\mu^T p = \mu^T$ is satisfied. Note that we must have 
\begin{equation}
     q(y^{m+1}) = \mu_{y_0} \prod_{t = 1}^{m} p_{y_{t - 1} y_{t}}.
     \label{eqn:flow_constraint}
\end{equation}
The distribution $q$ thus obtained is  \textit{stationary Markov}, since the flow constraint is also incorporated.

The expected reward function obtained from simplifying~\eqref{eqn:expected_reward_expression} using~\eqref{eqn:flow_constraint} is
\begin{multline}
    \EE_{Y^{m+1} \sim q}[r(Y^{m+1})] = \sum_{y^{m+1} \in \YY^{m+1}} q(y^{m+1}) r(y^{m+1}) \\= \sum_{t = 1}^m \sum_{\substack{(a, b) \in \\\YY \times \YY}} \mu_a p_{a b} r_{a b} = m \sum_{\substack{(a, b) \in \\\YY \times \YY}} \mu_a p_{a b} r_{a b}.
\end{multline}
Furthermore, the KL-constrained RL optimization constraint in \eqref{sec:KL-constrained_RL} can be rewritten (see, for instance, \citet[Theorem 1]{markovian_kl}) as
\begin{equation}
    \dKL(q, Q) = \dKL(\mu, \pi) + m \dKL(p, P) \le m \delta,
\end{equation}
where $\dKL(p, P)$ denotes the KL between two Markov chains with transition matrices $p$ and $P$ and with stationary distributions $\mu$ and $\pi$ respectively, defined as
$
    \dKL(p, P) = \sum_{i, j \in \YY} \mu_i p_{ij} \log \left( \frac{p_{ij}}{P_{ij}} \right).
$
Equivalently, $\dKL(p, P) \le \delta + o \left(1 \right)$. We neglect the $o(1)$ term in our constraint and set our constraint to be $\dKL(p, P) \le \delta$.

This gives us the following optimization problem ($\Delta_{a \times a}$ denotes the set of $a \times a$ stochastic matrices):
\begin{multline}
\max_{\substack{\mu \in \Delta_{|\YY|}, \\ p \in \Delta_{|\YY| \times |\YY|}}} \sum_{(a, b) \in \YY \times \YY} \mu_a p_{a b} r_{a b} \\  \text{ subject to } \dKL(p, P) \le \delta,\ \mu^T p = \mu^T. \label{eqn:optimization_problem_with_flow} \tag{KL-RL with Flow}
\end{multline}

Recall the definition of $G_{\beta}$ from Theorem~\ref{thm:q_expression}. The following theorem gives the solution to the optimization problem above, where we crucially use that $P$ is aperiodic and irreducible (Assumption~\ref{assump:Markovness}) to obtain the Perron eigenvalue \citep[Page 674, Equation~8.3.10]{meyer_linear_algebra} and the corresponding eigenvector of $G_{\beta}$.
\begin{thm}
The solution to~\eqref{eqn:optimization_problem_with_flow} is
\begin{equation}
    p_{ab} = \frac{{(G_{\beta})}_{ab} h_b}{\lambda_\beta h_a}
\end{equation} and
$\mu$ is the (unique) solution to $\mu^T p = \mu^T$, where $\lambda_\beta$ is the Perron eigenvalue of $(G_{\beta})$ and $h$ is the corresponding right Perron eigenvector, i.e., 
    $G_\beta h = \lambda_\beta h$, with  $\beta$ chosen such that $\dKL(p, P) = \delta$.
\label{thm:p_expression}
\end{thm}
The proof of this theorem involves working with the Lagrangian and is relegated to Appendix~\ref{appendix:p_expression}.

Next, suppose that the value of the solution to the optimization problem~\eqref{eqn:optimization_problem_with_flow} is $s_{\delta}$. We provide an equivalent expression for $s_{\delta}$ in the following theorem. In particular, for a fixed $\delta>0$, the value $s_{\delta}$ is related to the rate of change of $\log(\lambda_{\beta})$ with respect to $\beta$, where $\beta$ is chosen as in Theorem~\ref{thm:p_expression}. 
\begin{thm}
    Let $\Lambda_{\beta} = \log(\lambda_{\beta})$. Then, $s_{\delta} = \frac{\partial \Lambda_{\beta}}{\partial \beta} = \Lambda'_\beta$, where $\beta$ is as in Theorem~\ref{thm:p_expression}.
    \label{thm:s_delta_value}
\end{thm}
The proof of this theorem involves taking the derivative of $\Lambda_\beta$ using the multivariable chain rule and is relegated to Appendix~\ref{appendix:s_delta_value}.

Next, we state the following lemma, which follows from the Perron-Frobenius Theorem (see, for instance~\citet[Page~674, Equation~8.3.10]{meyer_linear_algebra}) and the proof of which is relegated to Appendix~\ref{appendix:G_beta_perron_frobenius}. Let $g$ be the left Perron eigenvector of $G_{\beta}$, where the magnitude of $g$ is taken such that $g^T h = 1$. 
\begin{lem}
We have $G_{\beta}^n = \lambda_\beta^n h g^T + o(\lambda_\beta^n).$
\label{lem:G_beta_perron_frobenius}
\end{lem}

Finally, we provide the full version of Theorem~\ref{thm:p_convergence_informal}, relating the matrices $p^{(t)}$ arising from the optimization problem without the flow constraint (i.e.~\eqref{eqn:optimization_problem}) and the matrix $p$, arising from the optimization problem with the flow constraint (i.e.~\eqref{eqn:optimization_problem_with_flow}). We show that the matrices $p^{(t)}$ converge to $p$ as $m \to \infty$ in the following theorem, the proof of which uses Lemma~\ref{lem:G_beta_perron_frobenius} and is relegated to Appendix~\ref{appendix:p_relation}.
\begin{thm}
For $1 \le t \le m$ such that $\lim_{m \to \infty} (m - t) \to \infty$, we have
$\lim_{m \to \infty} p^{(t)} \to p.$
\label{thm:p_relation}
\end{thm}

\section{LDP for average reward}
\label{sec:LDP_average_reward}
In this section, we prove Theorem~\ref{thm:rate_kl_value}. Among the $m$ transitions in the $m + 1$ samples, let the random variable $\theta_{ij}$ denote the empirical fraction of transitions from $i$ to $j$, where $i,j\in \YY$, i.e.,
\begin{equation}
\theta_{ij} = \frac{\sum_{t = 1}^m \on_{Y_{t - 1} = i, Y_{t} = j}}{m},
\end{equation}
and let $\theta_i = \sum_{j \in \mathcal{Y}} \theta_{ij}$. Recall that $b_{ij} = \frac{\theta_{ij}}{\theta_i}$. Note that $\theta_j \neq \sum_{i \in \YY} \theta_{ij}$ in general.

The theorem below establishes an LDP for the empirical transitions probabilities $\theta:= \{\theta_{ij}\}$ (i.e. a \textit{Markovian Sanov's theorem}), where we note that we get a non-trivial rate only if $\theta_j = \sum_{i \in \YY} \theta_{ij},\ \forall j \in \YY$.
\begin{thm}[\citet{markov_LDP}, Section~3.1.3]
$\theta$ has an LDP with rate function:
\begin{equation}
    J(\theta) = \begin{cases}
        \dKL(b, P) &\text{ if }\theta_j = \sum_{i \in \YY} \theta_{ij},\ \forall j \in \YY, \\
        \infty, & \text{ otherwise}.
    \end{cases}
\end{equation}    
\label{thm:markov_ldp}
Further, $J(\theta)$ is a \textit{good} rate function.
\end{thm}
Noting that $S_m$ is a continuous function of $\{\theta_{ab}\}_{a, b \in \YY}$ and using the Contraction principle~\citep[Theorem 4.2.1]{LDP_contraction_principle} yields the desired result, the full proof of which is relegated to Appendix~\ref{appendix:rate_kl_value}. We remark that this procedure affords a simpler argument than the one used in~\citet{asymptotics_of_LM_alignment} for the setting of i.i.d. outputs as well.

The rate function obtained in Theorem~\ref{thm:rate_kl_value} is similar to the one in \eqref{eqn:optimization_problem_with_flow}, with the following important difference: the aim in~\eqref{eqn:optimization_problem_with_flow} is to maximize the expected reward subject to a KL constraint between the tilted and LM distribution, while in Theorem~\ref{thm:rate_kl_value}, the aim is to 
to minimize the KL divergence between the tilted and LM distributions for a given expected reward. Let $t_{\delta} \triangleq \max \{s: \tilde{J}(s) \le \delta \}$.

We show that $t_\delta$ equals $s_{\delta}$ in the following lemma, the proof of which is relegated to Appendix~\ref{appendix:s_t_equality}.
\begin{lem} For $s_{\delta}$ being the optimum value of the optimization problem in~\eqref{eqn:optimization_problem_with_flow},
$s_{\delta} = t_{\delta}$.
\label{lem:s_t_equality}
\end{lem}

\section{Connecting the KL-constrained RL and \bestofn{} solutions}
\label{sec:connection_KL_best-of-n}

In this section, we prove Theorem~\ref{thm:KL-constrained_best-of-n}. First, we use~\eqref{eqn:Q_definition} and the denominator from~\eqref{eqn:numerator_expression} to obtain
\begin{multline}
\label{eq:temp}
\dKL(\qbest,q) = \EE_{Y^{m+1}\sim\qbest} \left[ \log\frac{\qbest(Y^{m+1})}
              {q_\beta(Y^{m+1})} \right] \\
\shoveleft{= \EE_{\qbest} \left[\log\frac{\qbest(Y^{m+1})}{Q(Y^{m+1})} - \log\frac{q_\beta(Y^{m+1})}{Q(Y^{m+1})} \right]}
\\ \shoveleft{= \dKL(\qbest,Q) - \beta \EE_{\qbest} \left[r(Y^{m+1})\right]} \\
- \log\left(\pi^\top(G_\beta)^m\on_{|\YY|} \right).
\end{multline}
Note that the first term on the right-hand side is $\dKL(\qbest, Q) \le \log (n) - \frac{n-1}{n}$ from~\citet[Theorem 3.1]{best_of_n_policy}, which we recall is at most $m \delta$ since $n = \exp(m \delta)$. To analyze the third term in the right-hand side above, we use Lemma~\ref{lem:G_beta_perron_frobenius} to get
\begin{multline*}
    \lim_{m \to \infty} \frac{1}{m} \log\left(\pi^\top (G_\beta)^m \on_{|\YY|} \right) \\
    = \lim_{m \to \infty} \frac{\log \left(\lambda_\beta^m \cdot \left(\pi^\top h g^\top \on_{|\YY|} + o(1) \right) \right) }{m} = \log(\lambda_\beta).
\end{multline*}

It is left to analyze the second term in \eqref{eq:temp}, i.e. $\EE_{Z^{m+1} \sim \qbest} \left[r(Z^{m+1})\right]$. Let $T_{m, n}(Z^{m+1}) = \frac 1m r(Z^{m+1})$, where $Z^{m+1} \sim \qbest$. Recall that we have an LDP for $S_m(Y^{m+1})$ from Theorem~\ref{thm:rate_kl_value}, where $S_m$ is the average reward when $Y^{m+1}$ follows the distribution $Q$. However, \eqref{eq:temp} has the expectation in terms of the distribution $\qbest$ and thus it is necessary to connect the expected value of the average reward under $Q$ to that under $\qbest$.

By the definition of \bestofn{}, note that the reward $T_{m, n}$ is the maximum of the rewards of $n$ i.i.d.~reward samples obtained from $Q$. Formally, let $S_m^{(1)}, S_m^{(2)}, \dots S_m^{(n)}$ be $n$ i.i.d. copies of $S_m$. Then,
\begin{equation}
    T_{m, n} = \max_{i = 1}^n S_m^{(i)}.
    \label{eqn:T_definition}
\end{equation}
We are interested in $\EE[T_{m, n}]$ when $n = \exp(m \delta)$. First, we consider the random variable $T_{m, \exp(m \delta)}$ and establish its convergence in probability to $s_{\delta}$.
\begin{lem}
We have
$T_{m, \exp(m \delta)} \xrightarrow{\mathbb{P}} s_{\delta}$.  
\label{lem:T_m_convergence}
\end{lem}
The proof of this theorem uses the LDP rate function from Theorem~\ref{thm:rate_kl_value} and is relegated to Appendix~\ref{appendix:T_m_convergence}.

The following corollary on the expected value of $T_{m, \exp(m \delta)}$ holds since $T_{m, \exp(m \delta)}$ is bounded.

\begin{cor}
We have
$
    \EE[T_{m, \exp(m \delta)}] \xrightarrow{} s_{\delta}.
$
\end{cor}    

Now, we have shown that
\begin{equation}
    \lim_{m \to \infty} \frac 1m \dKL(\qbest, q) = \delta - \beta s_{\delta} + \log(\lambda_{\beta}).
\end{equation}
It is left to show that this expression is zero, which we do in the following lemma, the proof of which is algebraic and is relegated to Appendix~\ref{appendix:delta_lambda_relation}. Recall from Theorem~\ref{thm:s_delta_value} that $\Lambda_{\beta} = \log(\lambda_{\beta})$ and $s_{\delta} = \Lambda'_{\beta}$.
\begin{lem}
We have
$
    \delta - \beta \Lambda'_{\beta} + \Lambda_{\beta} = 0.
$
\label{lem:delta_lambda_relation}
\end{lem}

\section{Closeness of \bestofn{} and the KL-RL distribution for $m=1$}
\label{sec:small_m_connection_KL_best-of-n}
We first state some measure theoretic assumptions.  Assume that $\YY$ is a Borel space, equipped with its Borel $\sigma$-algebra $\BB(\YY)$, so that the conditional distributions given $L=r(Y)$ exist. Recall that $Q$ denotes the law of $Y$, and that $q_\beta$ denotes the tilted distribution satisfying the following Radon--Nikyodym derivative:
\begin{equation}
\frac{dq_\beta}{dQ}(y) = \frac{e^{\beta r(y)}}{A_\beta}, \label{eqn:rn_tilted}
\end{equation}
where $A_\beta$ is a suitable normalization constant. Assume that $r: \YY \to \RR$ is Borel measurable. Let
$p_L$ and $p_{L'}$ denote the laws of $L=r(Y)$ under $Q$ and $q_\beta$, respectively and let the support of $L$ be $\LL$.

We note the following lemma, where we use that the `tilt' (i.e. $\exp(\beta r(y))$ term) in $q_\beta$ depends only on $r(y)$ and the proof of which is relegated to Appendix~\ref{appendix:borel_change}. 
\begin{lem} For $B \in \BB (\mathcal Y)$, for $p_L'$-almost every $l$,
\begin{equation}
q_{\beta} (Y' \in B | L' = l) = Q (Y \in B | L = l)
\end{equation}
\label{lem:borel_change}
\end{lem}
\vspace{-0.6cm}
We note the corollary below about the KL divergences, which follows by using the chain rule for KL divergence \citep[Theorem 2.5.3]{elements_of_information_theory}.
\begin{cor}
We have that $\dKL(p_{L'}, p_{L}) = \dKL(q_\beta, Q)$.  
\label{corollary:KL_divergence}
\end{cor}
This corollary is useful in proving Theorem~\ref{thm:continuous_KL_separation}.

We will assume that $L$ has a continuous CDF $P_L(l), l \in \LL$ for simplicity\footnote{the result is true even without this assumption}. Recall the definitions of $F$ and $F^-$ for the \bestofn{} policy from~\citet{best_of_n_policy}:
\begin{align}
F(y) &= \PP_{Z \sim Q} (r(Z)\le r(y)), \quad y \in \YY \label{eqn:F_definition}\\
F^-(y) &= \PP_{Z \sim Q} (r(Z)<r(y)), \quad y \in \YY.
\end{align}
In our notation, $F(y) = P_L(r(y)),\ y\in \YY$. In~\citet[Lemma 2.3]{best_of_n_policy}, the probability mass function of the \bestofn{} policy was given to be $\qbestn(y) = (F(y))^n - (F^-(y))^n$. We next give the continuous analogue of this equation in the following lemma, the proof of which is relegated to Appendix~\ref{appendix:rn_best_of_n}. 
\begin{lem}
$\qbestn$ is absolutely continuous with respect to $Q$, with the Radon--Nikodym derivative $\frac{d \qbestn}{dQ}(y) = n P_L(r(y))^{n-1}$.
\label{lem:rn_best_of_n}
\end{lem}
Finally, we note that the KL--divergence between two distributions if zero if and only if they are equal almost surely and use the expressions for the \bestofn{} distribution from Lemma~\ref{lem:rn_best_of_n} and the tilted distribution from~\eqref{eqn:rn_tilted} to obtain Theorem~\ref{thm:continuous_KL_separation}, the proof of which is relegated to Appendix~\ref{appendix:continuous_KL_separation}.

\begin{remark} It appears that the reward function in Theorem~\ref{thm:continuous_KL_separation} is dependent on $\beta$ but the expressions for both $\qbestn$ and $q_\beta$ do not depend on $\beta$. This is since the value of $\beta$ and $r(y)$ are picked jointly, such that the effect of $\beta$ ``cancels'' in $q_\beta$.
\end{remark}

\subsection{Discrete Case} 
\textbf{Best-of-$n$ distribution:} Let \(y_i=i/d\) for $i \in [d]$. Note that the expression for $r(y)$ from Theorem~\ref{thm:continuous_KL_separation} is a strictly increasing in $y$. From this, 
\begin{align}
    \qbestn(y) &= y^n - \left(y - \frac 1d \right)^n, \ y \in \YY_d \\
    \qbestn(y_i) &= \frac{i^n-(i-1)^n}{d^n},\ i \in [d].
\end{align}

\textbf{Tilted Distribution:} We apply the definition of the tilted distribution  to obtain
\begin{align}
    q_\beta(y) & \propto 1 \cdot \exp(\beta r(y)) = \exp(\beta (b + \frac{(n - 1)}{\beta} \log(y))) \nonumber \\
    & \propto \exp(((n - 1)\log(y))) = y^{n - 1},\  y \in \YY_d, \\
    q_\beta(y_i) &= \frac{(i/d)^{n-1}}{\sum_{j=1}^d (j/d)^{n-1}} = \frac{i^{n-1}}{S_{d,n-1}},\ i \in [d],
\end{align}
where $S_{d,n-1} \triangleq \sum_{j=1}^d j^{n-1}$.

Note that $q_\beta$ in this case is actually independent of $\beta$. We verify below that the constraint $\dKL(q_\beta, Q) \le \delta$ is satisfied, the proof of which is in Appendix~\ref{appendix:dKL_discrete_Q}.
\begin{lem}
We have $\dKL(q_\beta, Q) \le \log(n) - \frac{n-1}{n} = \delta$.
\label{lem:dKL_discrete_Q}
\end{lem}
Finally, we compute 
\begin{multline*}
\dKL(\qbestn, q_\beta) \\= \sum_{i \in [d]} \frac{i^n-(i-1)^n}{d^n} \log \left(\frac{i^n-(i-1)^n}{d^n} \cdot \frac{S_{d,n-1}}{i^{n-1}} \right).
\end{multline*}
This is the exact finite-\(d\) expression, but we want an approximation of this to glean insights about the dependence on $d$ and $n$. Note that the numerator has an $i^{n - 1}$ term while the denominator has a term which can be Taylor-expanded to an $i^{n - 1}$ order term. We perform these expansions and cancel out corresponding terms to get Theorem~\ref{thm:discrete_KL_separation}, the full proof of which is in Appendix~\ref{appendix:discrete_KL_separation}.

\section{Conclusions}

\label{sec:conclusions}
In this work, we studied two popular Language Model (LM) alignment methods: KL-constrained RL, which, although optimal in terms of expected reward, requires full knowledge of the LM distribution, and \bestofn{}, which is a computationally inexpensive method widely adopted in practice. Assuming Markov LM outputs and local decomposition of rewards, we showed that the two aligned distributions are asymptotically close, as the length of the outputs increases to infinity. This answers an extension raised in~\cite{asymptotics_of_LM_alignment} (which in fact studied the same problem under i.i.d. outputs) to a more practical setting. We further identified classes of LM distributions and reward functions for which the two aligned distributions remain close in a finite, non-asymptotic regime, making progress on another question posed in~\cite{asymptotics_of_LM_alignment}. Future directions include analyzing asymptotic closeness under mixing assumptions and analyzing non-asymptotic closeness for arbitrary discrete distributions.

\section{AI Use Statement}

OpenAI’s GPT 5.5 model was used to generate the code for the \texttt{Python Matplotlib} plots in Fig.~\ref{fig:KL_separation_vs_d} and in cleaning up the proofs of Lemma~\ref{lem:T_m_convergence}, Lemma~\ref{lem:dKL_discrete_Q}, and Theorem~\ref{thm:discrete_KL_separation}. GPT 5.6 used to proof-read the paper and for grammatical and stylistic changes. The authors assume full responsibility for the contents of the paper.

\clearpage

\bibliographystyle{ieeetr}
\bibliography{refs}

\clearpage

\onecolumn
\appendices

\section{Proof of~Theorem~\ref{thm:q_simplified_expression}}
\label{appendix:q_expression}

\begin{proof}[Proof of Theorem~\ref{thm:q_simplified_expression}] First, consider the denominator of the expression from Theorem~\ref{thm:optimization_problem_q_solution}. Since the summation is over all sequences $z^{m + 1} \in \YY^{m+1}$, it is possible to write 
\begin{equation}
    \sum_{z^{m + 1} \in \YY^{m+1}} \pi_{z_0} \prod_{i=1}^m {(G_\beta)}_{z_{t - 1} z_t} \on_{z_t} = \pi^\top (G_\beta)^m \on_{| \YY |}. \label{eqn:numerator_expression}
\end{equation}
Next, consider the numerator and use $h_m(y_m) = 1$ as below:
\begin{align}
    \pi_{y_0} \prod_{t=1}^m {(G_\beta)}_{y_{t - 1} y_t} &= \pi_{y_0} \frac{(h_0)_{y_0}}{(h_0)_{y_0}} \prod_{t=1}^m {(G_\beta)}_{y_{t - 1} y_t} \frac{(h_t)_{y_t}}{(h_t)_{y_t}} \\
    &= \pi_{y_0} (h_0)_{y_0} \prod_{t=1}^m {(G_\beta)}_{y_{t - 1} y_t} \frac{(h_t)_{y_t}}{(h_{t - 1})_{y_{t - 1}}}\\
    &= \pi_{y_0} (h_0)_{y_0} \prod_{t=1}^m p^{(t)}_{y_{t - 1} y_t}.
\end{align}

Next, we verify that $\nu$ is a valid probability distribution. Clearly, non-negativity holds.
\begin{equation}
    \sum_{i \in \YY}\nu_i = \frac{\sum_{i \in \YY} \pi_i {h_i^{(0)}}}{\pi^\top (G_\beta)^m \on_{| \YY |}} = \frac{\sum_{i \in \YY} \pi_i {(G_{\beta}^m \on_{|\YY|})_i}}{\pi^\top (G_\beta)^m \on_{| \YY |}} = 1.
\end{equation}
Finally, we verify that the $p^{(t)}$'s are valid stochastic matrices, where we use the definition from~\eqref{eqn:h_t_definition}:
\begin{equation}
    \sum_{b \in \mathcal{Y}} p^{(t)}(a,b) = \sum_{b \in \mathcal{Y}} {(G_{\beta})}_{ab} \frac{h^{(t)}_b}{h^{(t - 1)}_a} = 1   
    \quad \forall 1 \le t \in m, \forall a \in \mathcal{Y}.
\end{equation}   
\end{proof}

\section{Proof of Theorem~\ref{thm:p_expression}}
\label{appendix:p_expression}

\begin{proof} We solve this problem by writing a Lagrangian for suitable constants $\beta, c_a, v_b, \eta$ for $a, b \in \YY$, as follows:
\begin{multline}
    \mathcal{L} = \sum_{(a, b) \in \YY \times \YY} \mu_a p_{a b} r_{a b} - \frac{1}{\beta} (\dKL(p, P) - \delta) + \sum_{a \in \YY} \mu_a c_a \left(\sum_{b \in \YY} p_{ab}-1\right) \\ + \sum_{b \in \YY} v_b \left(\sum_{a \in \YY} \mu_a p_{ab}-\mu_b\right) + \eta\left(\sum_a\mu_a-1\right).
\end{multline}
Taking gradients with respect to $p_{ab}$ first, we obtain
\begin{equation}
    \mu_a r_{ab} - \frac{1}{\beta} \mu_{a} \left(1 + \log\left( \frac{p_{ab}}{P_{ab}}\right)\right) + \mu_a c_a + v_b \mu_a = 0,
\end{equation}
thereby yielding
\begin{equation}
    p_{ab} = C_{a} D_b P_{ab} \exp(\beta r_{ab}) = C_{a} D_b (G_{\beta})_{ab}, \quad a, b \in \YY, 
\end{equation}
for suitable constants $C_{a}, D_{b} \ge 0$.

Recall from Assumption~\ref{assump:Markovness} that $P$ is aperiodic and irreducible. Using the definition ${(G_\beta)}_{ab} = P_{ab} \exp(\beta r_{ab})$, it follows that $G_{\beta}$ is primitive. From standard arguments (see, for instance~\cite[Page 674, Equation~8.3.10]{meyer_linear_algebra}), it also follows that $G_{\beta}$ will have a unique Perron eigenvalue $\lambda_{\beta}$, and a unique right Perron eigenvector (up to scaling) $h \in \RR^{|\YY|}$. Now, take $D_b = h_b$ and $C_a = \frac{1}{\lambda_{\beta} h_a}$ as in the theorem statement, where $h$ is defined as in the theorem statement. Such a $h$ would exist since $G$ is a non-negative matrix. 

It is left to verify all the constraints. First, we verify the row-stochasticity constraint:
\begin{equation}
    \sum_{b \in \YY} p_{ab} = \frac{(G_{\beta} h)_a}{\lambda_\beta h_a} = \frac{\lambda_\beta h_a}{\lambda_\beta h_a} = 1.
\end{equation}
The value of $\mu$ is chosen such that it is the left eigenvector of $p$, with the eigenvalue $1$, and such that $\sum_{a \in \YY} \mu_a = 1$ is satisfied. This choice of $\mu$ ensures that both these constraints are trivially satisfied. Finally, the value of $\beta$ is chosen such that $\dKL(p, P) = \delta$.

Next, we argue that $\mu$ is unique. Let $D_h = \diag(h_1, h_2, \dots, h_{|\YY|})$. From this, note that
\begin{equation}
    p = D_h^{-1} \frac{G_{\beta}}{\lambda_{\beta}} D_h.
\end{equation}
Thus, $p$ is similar to $\frac{G_{\beta}}{\lambda_{\beta}}$. Since $\frac{G_{\beta}}{\lambda_{\beta}}$ itself has a unique left Perron vector (up to scaling) corresponding to the eigenvalue $1$, $p$, via similarity with $\frac{G_{\beta}}{\lambda_{\beta}}$, also has an unique left Perron eigenvector (say $\mu$) corresponding to the eigenvalue $1$. $\mu$ will hence be that (unique) positive vector satisfying $\mu^T p = \mu^T$, whose elements  sum up to $1$. \end{proof}

\section{Proof of Theorem~\ref{thm:s_delta_value}}
\label{appendix:s_delta_value}

\begin{proof} Recall that $s_{\delta} = \sum_{(i, j) \in \YY \times \YY} \mu_i p_{i j} r_{i j}$, where $\mu$ and $p$ are given by Theorem~\ref{thm:p_expression}. Although $h$ (see Theorem~\ref{thm:p_expression}) is a function of $\beta$, we omit an explicit dependence on $\beta$ in the notation, for brevity. Let $g$ be the left eigenvector of $G_{\beta}$ corresponding to the eigenvalue $\lambda_{\beta}$, i.e. $g^T G_{\beta} = \lambda_{\beta} g^T$ and let its magnitude be such that $g^T h = 1$. 

First, we relate $\mu$ with $g$ and $h$. Take the quantity 
\begin{equation} 
    \sum_{i \in \YY} g_i h_i p_{ij} = \sum_{i \in \YY} g_i h_i \frac{(G_{\beta})_{ij} h_j}{\lambda_{\beta} h_i}
    = \frac{h_j}{\lambda_{\beta}} \sum_{i \in \YY} g_i (G_{\beta})_{ij} = g_j h_j.
\end{equation}
Also recall that $\sum_{j \in \YY} g_j h_j = 1$. Recall from Theorem~\ref{thm:p_expression} that $p$ has a unique left eigenvector $\mu$ corresponding to the eigenvalue $1$ whose elements also sum up to $1$. From the above discussion, we see that the vector $g_j h_j$ is the unique left eigenvector (with eigenvalue $1$) for $p$ whose elements sum up to $1$, i.e., we must have $\mu_j = g_j h_j$.

Now, differentiating the expression $G_{\beta} h = \lambda_{\beta} h$ with respect to $\beta$ on both sides, we get:
\begin{equation}
    G_{\beta} h' + G_{\beta}'h = \lambda_{\beta}' h + \lambda_{\beta} h',
\end{equation}
where $'$ denotes the derivative with respect to $\beta$. By left-multiplying both sides with $g^T$, we get
\begin{align}
    g^T G_{\beta} h' + g^T G_{\beta}'h &= \lambda_{\beta}' g^T h + \lambda_{\beta} g^T h'. 
\end{align}
Observing that  $g^T G_{\beta} = \lambda_\beta g^T$, we obtain
\begin{align}
    g^T G_{\beta}'h &= \lambda_{\beta}'.
\end{align}
Note that $(G_{\beta})'_{ij} = (G_{\beta})_{ij} r_{ij}$ from the definition in Theorem~\ref{thm:q_expression}. From this, we have
\begin{align}
    \lambda_{\beta}' &= \sum_{i, j \in \YY} g_i (G_{\beta})_{ij} r_{ij} h_j,
    \end{align}
and hence that
\begin{align}
    \Lambda_{\beta}' &= \frac{\lambda_{\beta}'}{\lambda_{\beta}} = \sum_{i, j \in \YY} \frac{g_i (G_{\beta})_{ij} r_{ij} h_j}{\lambda_{\beta}}.
\end{align}
Finally, we use the definition of $p_{ij}$ and $\mu_i = g_i h_i$ to get
\begin{equation}
    \Lambda_{\beta}' = \sum_{i, j \in \YY} g_i h_i r_{ij} p_{ij} = \sum_{i, j \in \YY} \mu_i r_{ij} p_{ij} = s_{\delta}.
\end{equation}
\end{proof}

\section{Proof of Lemma~\ref{lem:G_beta_perron_frobenius}}
\label{appendix:G_beta_perron_frobenius}

\begin{proof}
Recall the earlier argument from the proof of Theorem~\ref{thm:p_expression} that the Perron-Frobenius Theorem can be applied on the matrix $G_{\beta}$ and that it has a unique Perron eigenvalue $\lambda_{\beta}$, a unique left Perron eigenvector $g$, and a unique right Perron eigenvector $h$. From the Perron Frobenius theorem, if $g$ and $h$ are normalized such that $g^T h = 1$, we have
    \begin{equation}
        \lim_{n \to \infty} \frac{G_{\beta}^n}{\lambda_{\beta}^n} = h g^T.
    \end{equation}
    From this,
    \begin{equation}
        \frac{G_{\beta}^n}{\lambda_{\beta}^n} = h g^T + o(1).
    \end{equation}
Multiplying by $\lambda_{\beta}^n$ completes the proof.\end{proof}

\section{Proof of Theorem~\ref{thm:p_relation}}
\label{appendix:p_relation}

\begin{proof} 
Using Lemma~\ref{lem:G_beta_perron_frobenius}, we obtain
\begin{equation}
    h^{(t)} = G_{\beta}^{m - t} \on_{|\YY|} = \lambda_\beta^{m - t} (g^T \on_{|\YY|}) h + o(\lambda_\beta^{m - t})
\end{equation}
and
\begin{align}
    G_{\beta} h^{(t)} &= \lambda_\beta^{m - t} (g^T \on_{|\YY|}) G_\beta h + o(\lambda_\beta^{m - t + 1}) \\
    &= \lambda_\beta^{m - t + 1} (g^T \on_{|\YY|}) h + o(\lambda_\beta^{m - t + 1}).
\end{align}

From the above two equations, we get that for $a,b\in \YY$,
\begin{align}
    \lim_{m \to \infty} p^{(t)}_{a,b} &= \lim_{m \to \infty} {(G_{\beta})}_{ab} \frac{h^{(t)}_b}{h^{(t - 1)}_a} \\
    &= {(G_{\beta})}_{ab} \lim_{m \to \infty}  \frac{h^{(t)}_b}{(G_\beta h^{(t)})_a} \\
    &= {(G_{\beta})}_{ab} \lim_{m \to \infty} \frac{\lambda_\beta^{m - t} (g^T \on_{|\YY|}) h_b + o(\lambda_\beta^{m - t})}{\lambda_\beta^{m - t + 1} (g^T \on_{|\YY|}) h_a + o(\lambda_\beta^{m - t + 1})} \\
    &= {(G_{\beta})}_{ab} \frac{h_b}{\lambda_{\beta} h_a} = p_{ab}.
\end{align}

\end{proof}

\section{Proof of Theorem~\ref{thm:rate_kl_value}}
\label{appendix:rate_kl_value}

\begin{proof} Note that
\begin{equation}
    S_m(Y^{m + 1}) = \frac{1}{m}\cdot r(Y^{m+1}) = \sum_{a,b \in \YY} \theta_{ab} r_{ab}.
    \label{eqn:S_m_continuous_defn}
\end{equation}
From this, note that $S_m$ is a continuous function of $\{\theta_{ab}\}$. Let $S_m = \phi(\theta)$ for some continuous function $\phi$. Using the contraction principle~\citep[Theorem 4.2.1]{LDP_contraction_principle} and noting that $J(s)$ is a good rate function, $S_m$ also satisfies a LDP with rate
\begin{equation}
    \tilde{J} (s) = \min \{J(\theta): \phi(\theta) = s\}.
\end{equation}
In our case, using Theorem~\ref{thm:markov_ldp}, we obtain the rate in the theorem statement. Proving that $\tilde{J} (s)$ is a {good} rate function also follows from the contraction principle, since the rate function of a random variable obtained by a continuous map of a random variable with a good rate function is itself good (see, for instance, \cite[Theorem 4.2.1]{LDP_contraction_principle}). \end{proof}

\section{Proof of Lemma~\ref{lem:s_t_equality}}

\label{appendix:s_t_equality}
\begin{proof} First, note that $t_\delta$ can be written alternatively as below:
\begin{equation}
    t_{\delta} = \max \Bigg\{
        s: \min_{\substack{\mu \in \Delta_{|\YY|},\\p \in \Delta_{|\YY| \times |\YY|}}} \sum_{i,j \in \YY} \mu_i p_{ij} \log\left(\frac{p_{ij}}{P_{ij}}\right) \leq \delta
        \text{ subject to }\mu^T p = \mu^T, \sum_{i,j \in \YY} \mu_i p_{ij} r_{ij} = s
        \Bigg\}.
\end{equation}

In Lemma~\ref{thm:rate_kl_value}, $\theta_i$ and $b$ are `dummy' variables and we replace them respectively by $\mu$ and $p$:
\begin{equation}
    \tilde{J} (s) = \min_{\mu \in \Delta_{|\YY|}, p \in \Delta_{|\YY| \times |\YY|}} \left\{\sum_{i, j \in \YY} \mu_i p_{ij} \log \left(\frac{p_{ij}}{P_{ij}} \right) : \sum_{i, j \in \YY} r_{ij} \mu_i p_{ij} = s, \quad \sum_{i \in \YY} \mu_i p_{ij} = \mu_j \right\}.
\end{equation}

We show the result by proving $s_\delta \le t_\delta$ and $t_\delta \le s_\delta$. First, let
\((\mu^\star,p^\star)\) be an optimal solution to the optimization problem
in~\eqref{eqn:optimization_problem_with_flow}. Then,
\begin{equation}
    s_\delta = \sum_{i,j \in \YY} \mu_i^\star p_{ij}^\star r_{ij},
\end{equation}
and
\begin{equation}
    \sum_{i,j \in \YY} \mu_i^\star p_{ij}^\star \log\left(\frac{p_{ij}^\star}{P_{ij}}\right) \leq \delta.
\end{equation}
Since \((\mu^\star,p^\star)\) is feasible for the minimization defining
\(\tilde{J}(s_\delta)\), we have
\begin{equation}
    \tilde{J}(s_\delta) \leq \sum_{i,j \in \YY} \mu_i^\star p_{ij}^\star \log\left(\frac{p_{ij}^\star}{P_{ij}}\right) \leq \delta.
\end{equation}
Therefore, \(s_\delta\) is feasible in the definition of \(t_\delta\), and hence $s_\delta \leq t_\delta$.
    
Conversely, let \(s\) be such that
\begin{equation}
    \tilde{J}(s) \leq \delta.
\end{equation}
By the definition of \(\tilde{J}(s)\), there exists a feasible pair
\((\mu^{(s)},p^{(s)})\) satisfying
\begin{equation}
    (\mu^{(s)})^T p^{(s)} = (\mu^{(s)})^T,
\end{equation}
\begin{equation}
    \sum_{i,j \in \YY} \mu_i^{(s)}p_{ij}^{(s)}r_{ij} = s,
\end{equation}
and
\begin{equation}
    \sum_{i,j \in \YY} \mu_i^{(s)}p_{ij}^{(s)} \log\left(\frac{p_{ij}^{(s)}}{P_{ij}}\right) \leq \delta.
\end{equation}
Thus, \((\mu^{(s)},p^{(s)})\) is feasible for the optimization problem
in~\eqref{eqn:optimization_problem_with_flow}. By the definition of
\(s_\delta\), it follows that $s \leq s_\delta$. Since this holds for every \(s\) satisfying \(\tilde{J}(s) \leq \delta\),
we obtain $t_\delta \leq s_\delta$.

Combining the two inequalities gives the desired result.
\end{proof}

\section{Proof of Lemma~\ref{lem:T_m_convergence}}
\label{appendix:T_m_convergence}

\begin{proof}
First, we consider the probability $\PP(T_{m, \exp(m \delta)} \ge s_{\delta} + \epsilon)$ for some $\epsilon > 0$. Note from~\eqref{eqn:T_definition}, using the union bound, and using the LDP for $S_m$ from Theorem~\ref{thm:rate_kl_value} that
\begin{multline}
    \PP(T_{m, \exp(m \delta)} \ge s_{\delta} + \epsilon) \le n \PP(S_m \ge s_{\delta} + \epsilon) \\ \le n \exp(-m \inf_{s \ge s_{\delta} + \epsilon} \tilde{J} (s) + o(m)) = \exp(m (\delta - \inf_{s \ge s_{\delta} + \epsilon} \tilde{J} (s)) + o(m)).
\end{multline}
Recall from Lemma~\ref{lem:s_t_equality} that $s_{\delta} = \sup \{s: \tilde{J} (s) \le \delta\}$. From this, note that for $s > s_{\delta}$, we necessarily have $\tilde{J}(s) > \delta$. Thus, for $s \ge s_{\delta} + \epsilon > s_{\delta}$, we have $\tilde{J}(s) > \delta$. Thus,
\begin{equation}
    \delta - \inf_{s \ge s_{\delta} + \epsilon} \tilde{J} (s) < 0
\end{equation}
and
$\lim_{m \to \infty} \PP(T_{m, \exp(m \delta)} \ge s_{\delta} + \epsilon) = 0$.

Next, consider $\PP(T_{m, \exp(m \delta)} \le s_{\delta} - \epsilon)$. Note that
\begin{equation}
    \PP(T_{m, \exp(m \delta)} \le s_{\delta} - \epsilon) = (1 - \PP(S_m > s_{\delta} - \epsilon))^n \le \exp(- n \PP(S_m > s_{\delta} - \epsilon)).
    \label{eqn:temp_appendix}
\end{equation}

We next argue that $\tilde{J}(s)$ is a convex function of $s$. To see this, first note that $J(\theta)$ itself is convex since the Markovian KL divergence is convex in the first argument. Next, recall that the map from $\theta$ to $S_m$ is a linear map as in~\eqref{eqn:S_m_continuous_defn}. Now, consider the function
\begin{equation}
    L(s, \theta) = \begin{cases}
    J(\theta), &\phi(\theta) = s, \\
    \infty, &\text{ otherwise.}
\end{cases}.
\end{equation}
Note that $L(s, \theta)$ is convex since $J(\theta)$ is convex. Finally, noting that $\tilde{J}(s) = \min_\theta L(s, \theta)$ gives the desired result that $\tilde{J}(s)$ is convex.

Recall that $\tilde{J}(s_\delta) \le \delta$.
\begin{enumerate}
    \item If $\tilde{J}(s_\delta) < \delta$, note that $\inf_{s > s_{\delta} - \epsilon} \tilde{J}(s) \le \tilde{J}(s_\delta) = \delta - \gamma$ for some $\gamma > 0$.
    \item The other case has $\tilde{J}(s_\delta) = \delta$.
    Since $\tilde{J}(s)$ is a rate function, we know that at the mean value $s_0 = \EE[S_m]$, we have $\tilde{J}(s_0) = 0$. Now, there are two possibilities:
\begin{enumerate}
    \item If $s_0 > s_\delta - \epsilon$, we have $\tilde{J}(s_0) = \delta - \gamma$ where $\gamma = \delta > 0$. Thus, we have $\inf_{s > s_{\delta} - \epsilon} \tilde{J}(s) \le \tilde{J}(s_0) = \delta - \gamma$.
    \item If $s_0 \le s_\delta - \epsilon$, suppose $s_0 = s_\delta - \beta$ for some $\beta \ge \epsilon$. Now, consider an $s_\delta - \Delta \in (s_\delta - \epsilon, s_\delta)$, where $\Delta < \epsilon$. From the convexity of $\tilde{J}$ and since $0 < \frac{\Delta}{\beta}< 1$, we know that
    \begin{equation}
        \tilde{J}(s_\delta - \Delta) = \tilde{J} \left( \left(1 - \frac{\Delta}{\beta} \right)s_\delta + \frac{\Delta}{\beta}(s_0) \right) \le \left(1 - \frac{\Delta}{\beta} \right) \delta + \frac{\Delta}{\beta} \tilde{J}(s_0) < \delta.
    \end{equation}
    Suppose that $\tilde{J}(s_\delta - \Delta) = \delta - \gamma$ for some $\gamma > 0$. Thus, we have $\inf_{s > s_{\delta} - \epsilon} \tilde{J}(s) \le \tilde{J}(s_\delta - \Delta) = \delta - \gamma$.
\end{enumerate}
\end{enumerate}

In all three of the cases, we have $\inf_{s > s_{\delta} - \epsilon} \tilde{J}(s) \le \delta - \gamma$ for some $\gamma > 0$ and thus
\begin{align}
\PP(S_m > s_{\delta} - \epsilon) &\ge \exp \left(-m \inf_{s > s_{\delta} - \epsilon} \tilde{J}(s) + o(m) \right) \nonumber \\
&\ge \exp(-m (\delta - \gamma) + o(m)).
\end{align}
From \eqref{eqn:temp_appendix}, we hence have
\begin{equation}
    \PP(T_{m, \exp(m \delta)} \le s_{\delta} - \epsilon) \le \exp(- \exp(m \delta) \exp(- m (\delta - \gamma) + o(m))) \le \exp(-\exp( m \gamma + o(m))).
\end{equation}
Taking the limit as $m\to \infty$ on both sides above gives the desired result.

\end{proof}

\section{Proof of Lemma~\ref{lem:delta_lambda_relation}}
\label{appendix:delta_lambda_relation}

\begin{proof}
Recall from Theorem~\ref{thm:p_expression} that
\begin{equation}
p_{ij} = \frac{G_{ij}h_j}{\lambda h_i} = \frac{P_{ij}\exp(\beta R_{ij})h_j}{\lambda h_i}.   
\end{equation}
From this,
\begin{equation}
\log\left(\frac{p_{ij}}{P_{ij}}\right) = \beta R_{ij} + \log h_j - \log h_i -
\log\lambda.    
\end{equation}
Recall that $\mu$ is the stationary distribution for $p$. Hence,
\begin{align}
\dKL(p,P) &= \sum_{i,j \in \YY} \mu_i p_{ij} \log\left(\frac{p_{ij}}{P_{ij}}\right) \\
&= \sum_{i,j\in \YY} \mu_i p_{ij} \left(\beta R_{ij} + \log h_j - \log h_i - \log\lambda \right) \\
&= \beta\sum_{i,j \in \YY }\mu_i p_{ij}R_{ij} + \sum_{i,j \in \YY}\mu_i p_{ij}\log h_j - \sum_{i,j \in \YY}\mu_i p_{ij}\log h_i - \log\lambda. \label{eq:temp2}
\end{align}
Since $\mu$ is stationary for $p$,
\begin{equation}
\sum_{i \in \YY} \mu_i p_{ij}=\mu_j \forall j \in \YY.
\end{equation}
Thus, since
\begin{equation}
\sum_{i,j \in \YY}\mu_i p_{ij}\log h_j = \sum_{j \in \YY}\mu_j\log h_j; \qquad \sum_{i,j \in \YY} \mu_i p_{ij}\log h_i = \sum_{i \in \YY} \mu_i\log h_i,
\end{equation}
and the second and third terms in \eqref{eq:temp2} cancel. Putting everything together, we obtain
\begin{align}
\dKL(p,P) &= \beta\sum_{i,j \in \YY}\mu_i p_{ij}R_{ij} - \log\lambda \\
&= \beta s_\delta - \Lambda.
\end{align}
Since $s_\delta=\Lambda'$ and since the value of $\beta$ is chosen so that $\dKL(p,P)=\delta$ as in~\eqref{eqn:beta_value_choice}, the desired result follows. \end{proof}

\section{Proof of Lemma~\ref{lem:borel_change}}
\label{appendix:borel_change}
\begin{proof} Let $g(l)=\frac{e^{\beta l}}{A_\beta}$,
so that
\begin{equation}
\frac{dq_\beta}{dQ}(y)=g(r(y)).    
\end{equation}

Let $B \in \BB (\mathcal Y)$, and let $C \in \BB(\RR)$ be a Borel subset of $\RR$. Then
\begin{equation}
q_\beta(Y'\in B, L'\in C) = \EE_{q_\beta} [\on_{Y' \in B} \on_{L' \in C}] = \EE_Q [\on_{Y \in B} \on_{L \in C} g(L)].    
\end{equation}
Next, use the tower law of expectations, conditioning on $L$:
\begin{equation}
q_\beta(Y'\in B, L'\in C) = \EE_Q [\EE_Q[\on_{Y \in B} \on_{L \in C} g(L) | L]] = \EE_Q [\on_{L \in C} g(L) \EE_Q[\on_{Y \in B} | L]] = \EE_Q [\on_{L \in C} g(L) Q(Y \in B | L)]. \label{eqn:closeness_continuous_proof_1}
\end{equation}
Now, take $B = \YY$ to obtain
\begin{equation}
    p_{L'}(C) = \EE_Q [\on_{L \in C} g(L)],
\end{equation}
from which we obtain
\begin{equation}
    \frac{dp_{L'}}{dp_{L}}(l) = \frac{e^{\beta l}}{A_{\beta}} = g(l).
\end{equation}

Finally, note that
\begin{align}
    q_\beta (Y' \in B | L' = l) &= \frac{q_\beta (Y' \in B, L' = l)}{q_\beta (L' = l)} \\
    &= \frac{\EE_Q[\on_{Y \in B, L = l} g(l)]}{\EE_Q[\on_{L = l} g(l)]} \\
    &= \frac{\EE_Q[\on_{Y \in B, L = l}]}{\EE_Q[\on_{L = l}]} \\
    &= \frac{Q(Y \in B, L = l)}{Q(L = l)} \\
    &= Q(Y \in B | L = l).
\end{align}
\end{proof}

\section{Proof of Lemma~\ref{lem:rn_best_of_n}}
\label{appendix:rn_best_of_n}
\begin{proof} Let $B\in\mathcal B(\mathcal Y)$ be a Borel set. By symmetry of the $n$ samples,
\begin{align}
\qbestn(B) &= \sum_{i=1}^n \PP \left(Y_i\in B,\ r(Y_i)>\max_{j\ne i}r(Y_j) \right)\\
&= n\PP\left(Y_1\in B, r(Y_j)<r(Y_1),\  \forall j=2,\ldots,n \right).
\end{align}
Conditioning on $Y_1$, using independence, and using~\eqref{eqn:F_definition},
\begin{align}
\qbestn(B) &= n \EE_Q\left[\on_{\{Y\in B\}} \PP\left(r(Y_j)<r(Y),\ \forall j=2,\ldots,n | Y \right) \right]\\
&= n \EE_Q \left[\on_{\{Y\in B\}} P_L(r(Y))^{n-1} \right].
\end{align}
Since this holds for every Borel set $B\in\mathcal B(\mathcal Y)$, the desired result follows.
\end{proof}

\section{Proof of Theorem~\ref{thm:continuous_KL_separation}}
\label{appendix:continuous_KL_separation}
\begin{proof} Assume that $\qbestn = q_\beta$. Since both distributions are absolutely continuous with respect to $Q$, their Radon--Nikodym derivatives with respect to $Q$ must agree almost surely. Hence,
by~\eqref{eqn:rn_tilted} and Lemma~\eqref{lem:rn_best_of_n},
\begin{equation}
n P_L(r(y))^{n-1} = \frac{\exp(\beta r(y))}{A_\beta}
\end{equation}
almost surely. Equivalently,
\begin{equation}
n P_L(l)^{n-1} = \frac{\exp(\beta l)}{A_\beta}
\end{equation}
almost surely. From this,
\begin{equation}
P_L(l) = \exp\left( \frac{\beta(l-b)}{n-1} \right)
\end{equation}
almost surely, where $b=\frac{1}{\beta}\log(nA_\beta) > 0$.

Since $\exp\left( \frac{\beta(l-b)}{n-1} \right)$ is a continuous function of $l$ and it is the CDF of the random variable $L$, it must be that $\LL$, the support of $L$, is also continuous. Further, since $P_L$ is a CDF, it must have a value of $1$ for $b \ge l$. From this, we get $\LL = (-\infty,b]$. This proves the forward direction of the theorem.

Conversely, suppose that $L$ has the CDF as in the theorem statement. Its density is
\begin{equation}
p_L(l) = \lambda\exp(\lambda(l-b))
\mathbf 1_{\{l\le b\}},
\end{equation}
where $\lambda= \frac{\beta}{n - 1}$. Then,
\begin{equation}
A_\beta = \EE_Q[\exp(\beta L)] = \frac{\lambda}{\lambda+\beta}\exp(\beta b) = \frac{\exp(\beta b)}{n}.
\end{equation}
Therefore,
\begin{equation}
\frac{dq_\beta}{dQ}(y) = \frac{\exp(\beta r(y))}{A_\beta} = n\exp(\beta(r(y)-b)) = nP_L(r(y))^{n-1} = \frac{d \qbestn}{dQ}(y)
\end{equation}
almost surely. Hence, the converse also follows.

It remains to compute the KL divergence between the tilted distribution $q_\beta$ and the distribution $Q$ to ensure that the condition that $\dKL(q_\beta, Q) \le m \delta = \delta$ is satisfied. First, note from~\eqref{eqn:rn_tilted} and this theorem statement that $L'$ also follows a reverse exponential distribution. Namely, for a suitable constant $C > 0$,
\begin{align}
    p_{L'}(l) &= C \exp \left(\frac{\beta}{n - 1}(l-b) \right) \exp(\beta l), \qquad l\le b \\
    &= \exp \left(\frac{\beta n}{n - 1}(l-b) \right), \, \qquad \qquad \qquad l\le b.
\end{align}
Clearly, $L'$ also follows a reverse exponential distribution with parameter $\frac{\beta n}{n - 1}$.

Next, note from Corollary~\ref{corollary:KL_divergence} and using the expression for KL-divergence between two exponential distributions with parameters $\frac{\beta n}{n - 1}$ and $\frac{\beta}{n - 1}$ (see, for instance, \cite[Table 1]{kl_exponential_distributions}) that
\begin{equation}
    \dKL(q_\beta, Q) = \dKL(p_{L'}, p_L) = \log \left(\frac{n - 1}{\beta} \cdot \frac{\beta n}{n - 1} \right) + \frac 1n - 1 = \log \left(n \right) + \frac {n - 1}n.
\end{equation}

Recalling that the value of $\beta$ is chosen such that $\dKL(q_\beta, Q) = \delta$ from~\eqref{eqn:beta_value_choice} gives the desired result.
\end{proof}

\section{Proof of Lemma~\ref{lem:dKL_discrete_Q}}
\label{appendix:dKL_discrete_Q}
First, we will show that the vector of probabilities $q_\beta$ is majorized by $\qbestn$. Recall that $q_\beta(y_i) = \frac{i^{n-1}}{S_{d,n-1}}$ and $\qbestn(y_i) = \frac{i^n - (i-1)^n}{d^n} \forall i \in [d]$.
Thus,
\begin{equation}
    r_i = \frac{\qbestn(y_i)}{q_\beta(y_i)} = \frac{S_{d, n - 1}}{d^n}\cdot \frac{i^n - (i-1)^n}{i^{n-1}} = \frac{S_{d, n - 1}i}{d^n} \cdot \left( 1 - \left(1- \frac 1i \right)^n \right)
\end{equation}
It can be seen that this is an increasing sequence with $i \in [d]$ (by, for instance, differentiating with respect to $i$). Also recall that $\sum_{i \in [d]} \qbestn(y_i) = \sum_{i \in [d]} q_\beta(y_i) = 1$. Thus, there exists an index $i_0$ such that 
\begin{equation}
    \qbestn(y_i) \le q_\beta(y_i),\ \forall i \le i_0 \qquad \text{and}\qquad    \qbestn(y_i) > q_\beta(y_i),\ \forall i > i_0.
\end{equation}
This means that the prefix sums (sum of $k$ largest elements) satisfy $\sum_{i = 1}^k \qbestn(y_{d - i}) \ge \sum_{i = 1}^k q_\beta(y_{d - i}),\ \forall k \in [d]$, with the equality holding only when $k = d$. From this, $\qbestn$ majorizes $q_\beta$. By using Karamata's inequality (see, for instance,~\citet[Theorem 1]{karamata_inequality}) and noting that $\dKL$ is convex in the first argument, we obtain
\begin{equation}
    \dKL(q_\beta, Q) \le \dKL(\qbestn{}, Q). \label{eqn:proof_discrete_1.1}
\end{equation}

Now, we attempt to upper bound $\dKL(\qbestn{}, Q)$. Now, consider the function: $f(x) = \frac{\lceil dx \rceil}{d}$, which is essentially a function that takes a continuous distribution and discretizes it. Consider the two (continuous) distributions: $s(x) = n x^{n-1},\ 0 < x \le 1$ and $t(x) = 1,\ 0 < x \le 1$. Note that if $X \sim s$, then
\begin{equation}
    \PP \left(f(X) = \frac {i}{d} \right) = \PP \left(\frac{i - 1}{d} < X \le \frac{i}{d} \right) = \PP \left(X \le \frac{i}{d} \right) - \PP \left(X \le \frac{i - 1}{d} \right) = \frac{i^n - (i-1)^n}{d} = \qbestn{} (y_i)
\end{equation}
and thus $f(X) \sim \qbestn$. Similarly, if $Y \sim t$, then $f(Y) \sim Q$. Finally, we use the data processing inequality for KL--divergence \cite[Corollary 2.18]{polyanskiy_book} using the function `$f$' to obtain
\begin{equation}
    \dKL(\qbestn{}, Q) \le \dKL(s, t) = \int_{0^+}^1 n x^{n-1} \log(n x^{n-1}) dx.
    \label{eqn:proof_discrete_1.2}
\end{equation}
It is possible to solve this integration using integration by parts to obtain
\begin{equation}
    \dKL(s, t) = \log (n) - \frac{n-1}{n}.
    \label{eqn:proof_discrete_1.3}
\end{equation}

Combining equations~\eqref{eqn:proof_discrete_1.1},~\eqref{eqn:proof_discrete_1.2}, and~\eqref{eqn:proof_discrete_1.3} yields the desired result.

\section{Proof of Theorem~\ref{thm:discrete_KL_separation}}
\label{appendix:discrete_KL_separation}

\begin{proof} Recall that we need to calculate the below quantity for $n \ge 4$:
\begin{equation}
\dKL(\qbestn, q_\beta) = \sum_{i \in [d]} \frac{i^n-(i-1)^n}{d^n} \log \left(\frac{i^n-(i-1)^n}{d^n} \cdot \frac{S_{d,n-1}}{i^{n-1}} \right),
\end{equation}
where $S_{d,n-1}=\sum_{j=1}^d j^{n-1}$.

First, note from the binomial expansion that
\begin{align}
i^n-(i-1)^n &= n i^{n-1} - \frac{n (n-1)}{2} i^{n-2} + O(i^{n-3}), \\
\frac{i^n-(i-1)^n}{i^{n-1}} &= n - \frac{n (n-1)}{2 i} + O \left(\frac 1{i^{2}} \right).
\end{align}
Also note from Faulhaber's formula \citep{faulhaber} that
\begin{align}
    S_{d, n - 1} &= \frac{d^n}{n} + \frac 12 d^{n-1} + O(d^{n - 2}), \\
    \frac{S_{d, n - 1}}{d^{n}} &= \frac{1}{n} + \frac 1{2d} + O \left(\frac{1}{d^2} \right).
\end{align}
Thus, 
\begin{align}
    \left(\frac{i^n-(i-1)^n}{d^n} \cdot \frac{S_{d,n-1}}{i^{n-1}} \right) &= \left(n - \frac{n (n-1)}{2 i} + O \left(\frac 1{i^{2}} \right) \right) \cdot \left(\frac{1}{n} + \frac 1{2d} + O \left(\frac{1}{d^2} \right) \right), \\
    &= 1 + \underbrace{\frac{n}{2d} - \frac{n-1}{2i} + O \left(\frac{1}{i^2}\right)}_{\delta_i} \label{eqn:discrete_proof_1}.
\end{align}
Let 
\begin{align}
    \delta_i = \frac{\qbestn(y_i)}{q_\beta(y_i)} - 1 = \frac{n}{2d} - \frac{n-1}{2i} + O \left(\frac{1}{i^2} \right). \label{eqn:discrete_proof_3}
\end{align}

Recall that $\qbestn$ and $q_\beta$ are both distributions, and thus satisfy
\begin{equation}
    \sum_{i \in [d]} q_\beta (y_i) \delta_i = 0. \label{eqn:discrete_proof_2}
\end{equation}
Thus, we have
\begin{align}
    \dKL(\qbestn, q_\beta) = \sum_{i \in [d]} \qbestn(y_i) \log\left(1 + \delta_i\right) & \overset{(a)}{=} \sum_{i \in [d]} q_\beta(y_i) (1 + \delta_i) \log\left(1 + \delta_i\right) \\
    &\overset{(b)}{=} \sum_{i \in [d]} q_\beta(y_i) \left(\delta_i +\frac{\delta_i^2}{2} + O(\delta_i^3) \right) \\
    &\overset{(c)}{=} \sum_{i \in [d]} q_\beta(y_i) \left(\frac{\delta_i^2}{2} + O(\delta_i^3) \right).
\end{align}
where $(a)$ uses $\qbestn(y_i) = q_\beta(y_i) (1 + \delta_i) \forall i \in [d]$, $(b)$ uses $\log\left(1 + \delta_i\right) = \delta_i - \frac{\delta_i^2}{2} + O(\delta_i^3)$, and $(c)$ uses~\eqref{eqn:discrete_proof_2}.

Next, note from~\eqref{eqn:discrete_proof_3} that
\begin{equation}
    \delta_i^2 = \frac{n^2}{4d^2} + \frac{(n-1)^2}{4i^2} - \frac{2 n (n-1)}{4 i d} + O \left(\frac{1}{i^3}\right)
\end{equation}
and that
\begin{align}
    \sum_{i \in [d]} O(q_\beta(y_i) \delta_i^3) &= \sum_{i \in [d]} O \left( \frac{i^{n - 1}}{S_{d, n - 1}} \frac{1}{i^3} \right) = \sum_{i \in [d]} O \left(\frac{i^{n - 4}}{S_{d, n - 1}} \right) = O \left(\frac{S_{d, n-4}}{S_{d, n - 1}} \right) = O \left(\frac{1}{d^3} \right),
\end{align}
where the last step uses Faulhaber's formula \citep{faulhaber} again.

Next, we perform some algebraic manipulations and use Faulhaber's formula \citep{faulhaber} for the last time to get the desired result.
\begin{align}
    \dKL(\qbestn, q_\beta) &= \frac 12 \sum_{i \in [d]} \frac{i^{n - 1}}{S_{d, n - 1}} \left(\frac{n^2}{4d^2} + \frac{(n-1)^2}{4i^2} - \frac{n (n-1)}{2 i d}\right) + O \left(\frac{1}{d^3}\right) \\
    &= \frac 12 \sum_{i \in [d]} \frac{i^{n - 1}}{S_{d, n - 1}} \left(\frac{n^2}{4d^2} + \frac{(n-1)^2}{4i^2} - \frac{n (n-1)}{2 i d}\right) + O \left(\frac{1}{d^3}\right) \\
    &= \frac{S_{d, n - 1}}{S_{d, n - 1}} \frac{n^2}{8 d^2} + \frac{(n-1)^2 S_{d, n - 3}}{8 S_{d, n - 1}} - \frac{n (n-1) S_{d, n - 2}}{4 S_{d, n - 1} d} + O \left(\frac{1}{d^3}\right) \\
    &= \frac{n^2}{8d^2} + \frac{(n-1)^2n}{8 d^{2} (n-2)} - \frac{n^2}{4 d^2} + O \left(\frac{1}{d^3}\right) \\
    &= \frac{n}{(n-2) 8d^2} + O \left(\frac{1}{d^3}\right).
\end{align}
    
\end{proof}

\end{document}